\documentclass{article} 
\usepackage{iclr2027_conference,times}

\usepackage{amsmath,amsfonts,bm}

\def\eqref#1{equation~\ref{#1}}

\def\1{\bm{1}}

\DeclareMathAlphabet{\mathsfit}{\encodingdefault}{\sfdefault}{m}{sl}
\SetMathAlphabet{\mathsfit}{bold}{\encodingdefault}{\sfdefault}{bx}{n}

\def\gG{{\mathcal{G}}}

\def\gL{{\mathcal{L}}}

\def\gO{{\mathcal{O}}}

\def\gV{{\mathcal{V}}}

\newcommand{\E}{\mathbb{E}}

\newcommand{\R}{\mathbb{R}}

\newcommand{\Var}{\mathrm{Var}}

\DeclareMathOperator*{\argmin}{arg\,min}

\usepackage{amsmath, amssymb, mathtools, amsthm}

\usepackage{graphicx}
\usepackage{booktabs}
\usepackage{enumitem}
\usepackage{hyperref}
\usepackage{cleveref}
\usepackage{url}
\usepackage{wrapfig}

\usepackage{algorithm}
\usepackage{algpseudocode}

\usepackage[%
linewidth=1pt,
middlelinecolor= black,
middlelinewidth=0.4pt,
roundcorner=1pt,
topline = false,
rightline = false,
bottomline = false,
rightmargin=0pt,
skipbelow=0pt,
leftmargin=0cm,
innerleftmargin=0.25cm,
innerrightmargin=0pt,
innertopmargin=4pt,
innerbottommargin=0pt,
]{mdframed}

\title{Variance-Optimal Control Variates for Learning with Black-box Feedback}

\author{Zihao Zhao\\
School of Computational Science and Engineering\\
Georgia Institute of Technology\\
\texttt{zzhao628@gatech.edu}\\
\And
Shuhan Zhang\\
School of Computational Science and Engineering\\
Georgia Institute of Technology\\
\texttt{szhang3007@gatech.edu}\\
\And
Kai Wang\\
School of Computational Science and Engineering\\
Georgia Institute of Technology\\
\texttt{kwang692@gatech.edu}
}

\theoremstyle{plain}
\newtheorem{theorem}{Theorem}[section]
\newtheorem{proposition}[theorem]{Proposition}
\newtheorem{lemma}[theorem]{Lemma}
\newtheorem{corollary}[theorem]{Corollary}
\theoremstyle{definition}

\theoremstyle{remark}
\newtheorem{remark}[theorem]{Remark}
\theoremstyle{example}
\newtheorem{example}[theorem]{Example}

\crefname{lemma}{Lemma}{Lemmas}
\crefname{assumption}{Assumption}{Assumptions}
\crefname{prob}{Problem}{Problems}
\crefname{eq}{Equation}{Equations}
\crefname{ass}{Assumption}{Assumptions}
\crefname{assum}{Assumption}{Assumptions}
\crefname{remark}{Remark}{Remarks}

\newcommand{\ctx}[1]{z_{#1}}
 \newcommand{\ghat}[1]{\hat{g}_{#1}}

\iclrfinalcopy

\begin{document}

\maketitle

\begin{abstract}
Modern models increasingly learn through black-box oracles such as humans, optimization solvers, and external tools that provide feedback without exposing their internal mechanisms.
In such cases, gradients are often estimated using score-function (REINFORCE) estimators but can suffer from prohibitive variance.
A common remedy is to learn an (action-)value function as a control variate.
In this paper, we first observe that even an exact action-value function can be arbitrarily far from variance-optimal.
We show that this gap arises because the value function minimizes the noise in each action's own gradient term, while an action can still affect the rest of the gradient estimator through shared parameters.
A simple unbiased correction, at no extra oracle cost, can still reduce its variance by an arbitrarily large factor.
Motivated by this, we then prove that the residual variance can be decomposed exactly by actions with no cross terms.
This decomposition yields a closed-form \emph{variance-minimizing correction} for neural-network parameters, which can be computed by a simple projection.
Empirically, our correction consistently reduces the variance left by the value function and improves learning across all tasks. 
The source code for all experiments is available  \href{https://anonymous.4open.science/r/black_box_opt-B4C6/}{here}.
\end{abstract}
\section{Introduction}
Many learning systems optimize predictions through a component that is not differentiable. In general, a predictor reads a context and outputs a vector of parameters, and then a black-box oracle evaluates the parameters and returns only a scalar loss. In decision-focused learning (DFL)~\citep{elmachtoub2022smart,mandi2024decision}, for example, the predicted parameters define an optimization problem, and the loss is the regret of the decision made by the solver.
In large language model (LLM) agents, the predicted parameters can be a tool configuration, e.g., which Model Context Protocol (MCP) servers the agent allows to call~\citep{mcp2024}, and the loss is whether the agent then completes the task.
In all these cases, the downstream computation is opaque, so we have no access to its gradient, and each oracle call can potentially be expensive.


A standard approach is the score-function (REINFORCE) estimator~\citep{williams1992simple,sutton1999policy,mohamed2020monte}, which provides an unbiased gradient using only sampled black-box evaluations. Its main limitation is variance. When a single scalar outcome must train many discrete outputs, the resulting gradient can be extremely noisy, especially in long-horizon settings or high-dimensional discrete action spaces~\citep{tokui2017evaluating}, and provides only coarse credit across outputs. A common unbiased remedy is therefore to use an action-value function as an action-dependent control variate~\citep{gu2016muprop, liu2018actiondependent, wu2018variance, foerster2018counterfactual}. The value function replaces the shared evaluations with conditional information about actions, providing customized credit while preserving unbiasedness. It is also the object that exact local marginalization computes by brute force, at the price of one oracle call per value of every action~\citep{titsias2015local,tokui2017evaluating,liu2019rao}, and it is known to be variance-optimal when every action owns its own parameters~\citep{tokui2017evaluating}.

However, we show that using this action-value function alone does not solve the variance-minimization problem. Even when the action-value function is known exactly, its gradient estimator can be arbitrarily far from variance-optimal. A simple unbiased correction, at no extra oracle cost, can still reduce its variance by an arbitrarily large factor. The key reason is parameter sharing. The action-value function minimizes the variance of each action's gradient contribution in isolation. However, the final estimator is the sum of these contributions.
When actions share neural-network parameters, their gradient vectors overlap in the same parameter space, so the variance of the sum also depends on how the contributions interact. A control variate that is optimal for each term separately is therefore not necessarily optimal for their sum, since changing one term can increase or decrease its covariance with the others. The  action-value function does not account for these cross-action interactions.


To characterize and remove this remaining variance, we show that the residual noise left by the action-value function admits an exact action-wise decomposition whose components have no cross terms in variance. Thus, although the original gradient contributions interact through shared parameters, the residual variance can be \textbf{optimized separately} for each action. This decomposition directly yields a closed-form variance-minimizing correction to the action-value function. {We further show that the optimal correction is simply a projection onto the realizable control-variate subspace, \emph{without} learning an additional control-variate network or solving a large global covariance system.}
This gives a different view of value-function baselines, in which the action-value function solves the local credit-assignment problem but is not in general the variance-optimal endpoint. When shared-parameter effects are weak, little correction is needed; when they are strong, substantial variance can remain even after the value function is exact. Table~\ref{tab:comparison} summarizes the trade-offs among these approaches and our method.

Empirically, we evaluate this effect across three tasks, including both controlled settings and black-box learning with language-model predictors. We find that substantial variance can remain after applying the action-value function, and that our correction consistently removes a large fraction of this residual variance. These reductions translate into improved learning performance across all three tasks.

\begin{table}[t]
    \centering
    \small
    \setlength{\tabcolsep}{5.5pt}
    \renewcommand{\arraystretch}{1.15}
    \caption{Trade-offs for optimizing discrete predictions through an expensive black-box oracle. $d$ is the number of actions and $V$ the number of values each action can take. $^\dagger$Variance-optimal within the realizable family of~\Cref{thm:row_opt}.}
    \label{tab:comparison}
    \begin{tabular}{lcccc}
        \toprule
        & \textbf{REINFORCE}
        & \textbf{Exact local marg.}
        & \textbf{Learned surrogate}
        & \textbf{Ours} \\
        \midrule
        Black-box oracle calls
        & $\gO(1)$
        & $\gO(dV)$
        & $\gO(1)$
        & $\gO(1)$ \\
        Unbiased gradient
        & \checkmark
        & \checkmark
        & \texttimes
        & \checkmark \\
        Variance of estimator
        & High
        & Low
        & Low
        & Variance-optimized$^\dagger$ \\
        Credit assignment
        & Uniform
        & Customized
        & Customized
        & Customized \\
        \bottomrule
    \end{tabular}
\end{table}
\section{Related Works}
\paragraph{Gradient estimation with action-independent control variates.}
When the loss is not differentiable, the score-function estimator~\citep{williams1992simple} is the standard unbiased choice, and~\citet{silvestri2026score} recently apply it to DFL. The same estimator drives black-box variational inference, where the log-derivative trick estimates the gradient of the ELBO~\citep{pbj12, ranganath2014black}. Because the score
function has high variance, a long line of work reduces it with baselines, including the optimal constant baselines~\citep{weaver2001optimal, greensmith2004variance}, learned baselines~\citep{sutton1999policy, mnih2014neural}, and leave-one-out baselines~\citep{mnih2016variational, kool2019buy}. Group-relative baselines are now widely used for LLM reinforcement learning~\citep{shao2024deepseekmath, ahmadian2024back, liu2025understanding}, and~\citet{li2026the} derive the optimal token-level baseline for long-horizon LLM RL.
Close to us, \citet{cheng2020trajectory} decompose the variance along a trajectory and show that a per-state baseline leaves the terms that couple different time steps. However, these baselines do not depend on actions, so every coordinate still receives the same credit.

\paragraph{Gradient estimation with action-dependent control variates.} 
Action-dependent control variate removes the credit limitation, and they have been built from learned critics with an analytic correction term~\citep{gu2017qprop}, continuous relaxation~\citep{gu2016muprop, tucker2017rebar, grathwohl2018backpropagation}, Stein's identity~\citep{liu2018actiondependent}, factorized policies~\citep{wu2018variance, foerster2018counterfactual}, and discrete structures~\citep{titsias2022double, shi2022gradient}, although \citet{tucker2018mirage} show their gain over action-independent ones is small in practice. 

{These action-dependent baselines mainly propose new forms of control variate and fit them by regression or by minimizing the variance directly, while the jointly variance-minimizing choice remains less understood when gradient contributions across decisions are coupled.} Known optimality results are for baselines, where the optimal baseline has a closed form~\citep{weaver2001optimal, greensmith2004variance, li2026the}, and so does the optimal action-dependent one when the scores of different coordinates are orthogonal~\citep{wu2018variance}. We characterize the optimum over per-decision control variates without this assumption (\Cref{thm:opt_cv}), which says that a learned control variate should predict how much it can gain over the value function, and when it can gain nothing.

\paragraph{Exact marginalization.}
Another line of work sums out one variable at a time instead of sampling it~\citep{titsias2015local, tokui2017evaluating, liu2019rao, SSM23}, and a related family couples several samples of the same input to the same end~\citep{yin2018arm, dong2020disarm, dimitriev2021arms}. In our notation, these methods build $Q_j$ exactly, by evaluating the loss at \emph{every value of every variable}, which costs $dV$ calls per step and is prohibitive when one call is an agent run. \citet{tokui2017evaluating} show that this choice is variance-optimal when each variable has its own parameters, while the shared-parameter case has been left open~\citep{SSM23}. Our results answer this question by proving that with shared parameters, $Q$ is no longer optimal, and the optimal control variate has a closed form.

\section{Problem Statement}
\paragraph{Black-box learning setting.} For a context $x$, a stochastic predictor parameterized by $\theta\in \R^n$ samples a (discrete) parameter vector $b = (b_1, \dots, b_d)$, whose entries take values in a finite set $\mathcal V$ of size $V$, according to a known categorical distribution
$p(b|x, \theta) = \prod_{j = 1}^d p(b_j | z_j, x, \theta)$.
Here, $z_j$ denotes the context available when action $j$ is sampled. 

We consider two common cases in practice. For an \emph{autoregressive} predictor such as an LLM that outputs one action (decision/token) after another, $z_j = b_{<j}$. For a \emph{factorized} predictor, the coordinates of $b$ are independent given $x$; we fix an arbitrary order of the coordinates and again set $z_j = b_{<j}$, so that $p_\theta(b_j|z_j,x) = p_\theta(b_j|x)$. All results below are stated for $z_j = b_{<j}$ and therefore cover both cases; for a factorized predictor, the chosen order affects the variance of the estimators in \Cref{sec:method} but not their unbiasedness. We write $p_j(v)$ for $p_\theta(b_j=v|\ctx{j})$ throughout.

The sampled vector is then evaluated by a black-box oracle, and we optimize
$\min_{\theta} \gL(\theta) \coloneqq \E_{p(b|\theta)} [f(b)] =\E_{p(b|\theta)} [\mathrm{Eval} (\mathrm{Oracle} (b))]$.
In decision-focused learning, $\mathrm{Oracle}$ solves a downstream problem with the predicted parameter $b$, and $\mathrm{Eval}$ evaluates the decision under the true parameters, so $f$ is the decision regret. Our analysis only assumes that $f$ can be evaluated but not differentiated. We also assume the oracle to be deterministic; an oracle with independent evaluation noise only adds an irreducible variance. 

\paragraph{\textbf{REINFORCE}.}
Since the $\mathrm{Oracle}$ is a black-box, we can employ the score-function to express the gradient of the objective $g$ as
\begin{align}\label{eq:sf}
    g \coloneqq \nabla_\theta \gL (\theta) = \E_{b \sim p(\cdot | \theta)}[f(b) \sum_{j=1}^d s_j(b_j)], \quad s_j(v) \coloneqq \nabla_\theta \log p_j(v).
\end{align}
The single-sample REINFORCE estimator $\hat g_{\mathrm{SF}}\coloneqq f(b)\sum_j s_j(b_j)$ requires only one evaluation of $f$,
but it suffers from both high variance and coarse credit assignment, since the same sampled loss multiplies every score term. Both worsen as $d$ grows because if the sum has $d$ terms, the single scalar $f(b)$ has to explain all of them.

\paragraph{Parameter sharing.}
{We focus on the common case where different actions share model parameters. Then their score vectors can overlap, so reducing the variance of each action separately may not minimize variance of their sum. 

In this paper, our goal is to characterize the unbiased action-dependent control variate that minimizes the variance of the full gradient estimator, i.e., find an optimal control variate that solves
\begin{align*}
    \hat g \in \argmin_{\hat g \in \gG} \Var(\hat g), \quad \text{s.t.} \ \E[\hat g ] = g,
\end{align*}
where $\gG$ denotes the family unbiased score-function estimators with action-dependent control variates.
\Cref{sec:row} restricts $\gG$ further to the corrections a neural-network gradient can realize; all optimality statements in this paper refer to that family.
This is challenging because the optimum must account for interactions across actions rather than treating them independently.}


\section{Variance-optimal Control Variate}\label{sec:method}
We ask how much of the score-function variance can be removed without introducing bias. Our key observation is that the coordinate-wise optimality can be strictly (and even arbitrarily) suboptimal when parameters are shared across actions. Then, we characterize the remaining variance through prefix-wise residuals and derive a variance-optimal correction in closed form.

\subsection{The Action-Value Control Variate}
Consider the $j$-th term in~\Cref{eq:sf}. Let $b_{\mathrm{rem}(j)}$ denote the randomness not contained in $(\ctx{j},b_j)$. Applying the law of total expectation gives
\begin{align}\label{eq:Q}
    \E_b\big[f(b)s_j(b_j)\big]
    &= \E_{\ctx{j}}
       \E_{b_j|\ctx{j}}
       \E_{b_{\mathrm{rem}(j)}|\ctx{j},b_j}
       \Big[f(b)s_j(b_j)\Big] \notag \\
    &= \E_{\ctx{j}}
       \E_{b_j|\ctx{j}}
       \big[Q(\ctx{j},b_j)
           s_j(b_j)\big], \qquad Q(z_j, v) \coloneqq \E_{b_{\mathrm{rem(j)}}}\big[f(b)|\ctx{j},b_j=v\big].
\end{align}
Here, $Q$ can be regarded as the action-value function, which is widely used in RL for variance reduction.
Since $b_j$ is categorical,
$ g = \nabla_\theta \gL (\theta) = \sum_{j=1}^{d}\E_{\ctx{j}}[\sum_{v\in\mathcal V}p_j(v)Q(\ctx{j},v)s_j(v)]$.

The exact sweep for $v \in \gV$ is usually too expensive to compute, since it costs one oracle call for each of the $dV$ pairs $(j,v)$, so a common way is to replace $Q$ by a predictor $\hat Q(z_j, v) \approx Q(z_j, v)$ fitted from previous oracle evaluations. 
However, simply replacing $Q$ by $\hat Q$ in~\Cref{eq:Q} would introduce bias due to approximation error. Instead, we use $\hat Q$ as a control variate~\citep{hammersley1964monte}. 
By adding and subtracting the expectation, we get the following estimator
\begin{align}\label{eq:g_hat_Q}
    \ghat{\hat{Q}}\coloneqq \sum_{j=1}^{d}\Big[ \underbrace{\big(f(b)-\hat{Q}(\ctx{j},b_j)\big)s_j}_{\mathrm{term \ (a)}}
        + \underbrace{\sum\nolimits_{v\in\mathcal V}p_j(v)\hat{Q}(\ctx{j},v)s_j(v)}_{\mathrm{term \ (b)}} \Big].
\end{align}
Here, $\mathrm{term \ (a)}$ denotes the score-function with an action-dependent baseline, while $\mathrm{term \ (b)}$ is its correction term, given by the expectation of the baseline-weighted score over the current action.
The estimator pays exactly one oracle call per step, the same as REINFORCE, and the sweep it needs is over the \emph{fitted model} $\hat Q$, not over the oracle.
\begin{lemma}[Unbiasedness and coordinate-wise optimality]\label[lemma]{lem:q_coord_opt}
Given a fitted predictor $\hat Q$, \Cref{eq:g_hat_Q} is unbiased for any choice of $\hat Q$, i.e.,  $ \E[\ghat{\hat{Q}}]=\nabla_\theta\gL(\theta) = g$.
Moreover, writing $\ghat{\hat{Q}, j}$ for the $j$-th summand in~\Cref{eq:g_hat_Q}, the exact conditional mean $Q$ minimizes its conditional variance, i.e., $Q(z_j, \cdot) \in \argmin_{\hat{Q}} \Var(\ghat{\hat{Q}, j}|z_j)$.
\end{lemma}
See Appendix~\ref{app:pf_q_coord_opt} for proof. The lemma explains why $Q$ is the natural target for a learned action-dependent control variate. It removes all variance that can be removed when the index $j$ is optimized in isolation. This is also the sense in which prior work calls it optimal~\citep{tokui2017evaluating}. 
However, previous work did not consider the total variance of the gradient estimator \emph{sum} over $j$, and their analysis does not guarantee its optimality.

\subsection{Why $Q$ Is Not Variance-Optimal}
\Cref{lem:q_coord_opt} shows that $Q$ can minimize each action's own gradient term in isolation. This does not imply that the variance of the sum of all $d$ terms in~\Cref{eq:g_hat_Q} is minimal, because those terms are vectors in one shared parameter space and can be correlated. The following example shows the gap exists even in the simplest case with $d=2$; Appendix~\ref{app:arbitrary} turns it into a family in which the exact value function removes \emph{none} of the REINFORCE variance while the optimum removes \emph{all} of it.

\begin{mdframed}
\begin{example}\label[example]{ex:two_bern}
    Consider two binary decisions generated one after another, $b_1\sim\mathrm{Bernoulli}(p)$ and $b_2\mid b_1\sim\mathrm{Bernoulli}(p)$, with a single shared parameter $p=\sigma(\theta)$, where $\sigma(\cdot)$ denotes the sigmoid function, and let $f(b_1,b_2)=b_1b_2$. At $\theta=0$ we have $p=1/2$ and $\nabla_\theta p=1/4$, so $\nabla_\theta\gL(\theta)=\nabla_\theta\,p^2=1/4$, and the score function is $s_j(v)=\nabla_\theta\log p(v)=v-1/2$. With prefix conditioning, $z_1=\emptyset$ and $z_2=b_1$, the exact conditional means are
    $Q_1(v)=\E[f\mid b_1=v]=v/2$ and $Q_2(b_1,v)=b_1v$:
    $Q_1$ averages out the later decision $b_2$, while $Q_2$ conditions on the earlier one. Substituting into~\Cref{eq:g_hat_Q} and using $b_1(b_1-\tfrac12)=b_1/2$,
    \begin{align*}
        \ghat{Q}
        =\underbrace{\tfrac{b_1}{2}\big(b_2-\tfrac12\big)}_{\mathrm{term\ (a)}}
        +\underbrace{\tfrac18+\tfrac{b_1}{4}}_{\mathrm{term\ (b)}},
        \qquad \E[\ghat{Q}]=\tfrac14,\qquad
        \Var(\ghat{Q})=\tfrac{3}{64}>0.
    \end{align*}
    Both summands remain random even though $Q$ is exact. The first through $f-Q_1(b_1)=b_1(b_2-\tfrac12)$, which depends on the \emph{later} decision, and the second through its correction term $\sum_v p_2(v)Q_2(b_1,v)s_2(v)=b_1/4$, which depends on the \emph{earlier} one. Now perturb the exact conditional means by $\Tilde Q_1=Q_1+\tfrac14$ and $\Tilde Q_2=Q_2+\tfrac{b_1}{2}$. Neither perturbation depends on $v$, so term (b) is unchanged (because $\sum_v p_j(v)s_j(v)=0$), while in term (a) the perturbations cancel the remaining randomness exactly:
    \begin{align*}
        \ghat{\Tilde Q}=
        \underbrace{-\tfrac14\big(b_1-\tfrac12\big)}_{\mathrm{term\ (a)}}
        +\underbrace{\tfrac18+\tfrac{b_1}{4}}_{\mathrm{term\ (b)}}
        =\tfrac14,\qquad
        \Var(\ghat{\Tilde Q})=0<\Var(\ghat{Q}).
    \end{align*}
\end{example}
\end{mdframed}
The issue in the $\ghat{\hat Q}$ estimator in~\Cref{eq:g_hat_Q} is not an inaccurate value function, because even $Q$ is exact, its variance could still be suboptimal. Rather, $Q$ minimizes the variance of one coordinate at a time, while the quantity we ultimately care about is the variance of their sum in a shared parameter space, which could potentially be further optimized. We next characterize exactly what remains.

\subsection{A General Framework for Variance-Optimal Control Variates}
In this section, we introduce a general correction term $\Delta$ to capture the remaining noise in $\ghat{\hat{Q}}$ estimator and minimize the variance of the resulting gradient estimator.
\paragraph{General control variate.} 
Let $\Delta_j(z_j, v) \in \R^n$ be any vector-valued function. We center it over the current action to obtain a zero-mean function:
\begin{align}\label{eq:centered_rho}
    \Tilde \Delta_j \coloneqq \Delta_j(z_j, b_j) - \sum_v p_j(v) \Delta_j(z_j, v), \quad U_\Delta \coloneqq \sum_{j=1}^d \Tilde \Delta_j.
\end{align}
By construction, $\E[\Tilde \Delta_j | z_j] = 0$ and hence $\E[U_\Delta] = 0$. Therefore, we have a general estimator
\begin{align}
    \ghat{\Delta} \coloneqq \ghat{Q} - U_\Delta,
\end{align}
which is unbiased for any choice of $\Delta$. Note that $\Delta_j$ is a function of the sampled action, so subtracting it keeps the per-action credit assignment of \Cref{eq:g_hat_Q} and costs no additional oracle call.

\paragraph{Decomposition of remaining noises.} Let $e_Q \coloneqq \ghat{Q} - g$ denote the {remaining noise} after subtracting $Q$. We first decompose the residual $e_Q$ into different action time steps. For each action $j$ and value $v$, define the \emph{residual increment}
\begin{align}\label{eq:xi_def}
    \xi_j(\ctx{j}, v) \coloneqq \E[\ghat{Q} | \ctx{j}, b_j = v] - \E[\ghat{Q} | \ctx{j}].
\end{align}
Thus, $\xi_j(\ctx{j}, v)$ measures how much choosing $b_j=v$ can change the expected gradient, given everything decided before $j$.  
By construction, we have
\begin{align}\label{eq:xi_centered}
   \E_{b_j|\ctx{j}}[\xi_j(\ctx{j},b_j)]  = \sum_{v} p_j(v) \xi_j(\ctx{j}, v) = 0.
\end{align}
\begin{lemma}[Action-wise decomposition]\label[lemma]{lem:decomp}
Let $\ctx{j}=b_{<j}$. Along a sampled $b$, we can write
$e_Q
=\ghat{Q} - g = \sum_{j=1}^{d}
\xi_j(\ctx{j},b_j)$.
Moreover, for any $i<j$, we have
$\E[\xi_i(\ctx{i},b_i)^\top\xi_j(\ctx{j},b_j)]=0$.
Therefore, the variance of $\ghat{Q}$ is 
$\Var(\ghat{Q})
    =
    \sum_{j=1}^{d}
    \E\big[\|\xi_j(z_j,b_j)\|_2^2\big]$.
\end{lemma}
\vspace{-12pt}
\begin{proof}
Because $z_j=b_{<j}$ and $(z_j,b_j)=b_{\le j}$,
\begin{align*}
    \sum_{j=1}^{d}\xi_j(z_j,b_j)
    &=
    \sum_{j=1}^{d}
    \left(
        \E[\ghat{Q}| b_{\le j}]
        -
        \E[\ghat{Q}| b_{<j}]
    \right)
    =
    \ghat{Q}-\E[\ghat{Q}]
    =
    e_Q,
\end{align*}
where the sum telescopes, using that $\ghat{Q}$ is a function of $b=b_{\le d}$ alone, and that $\E[\ghat{Q}]=g$ by \Cref{lem:q_coord_opt}. For any $i<j$, $\xi_i(z_i,b_i)$ is determined by $b_{\le i}$ and hence by $z_j=b_{<j}$. Therefore,
$\E[\xi_i^\top\xi_j]
    =
    \E\left[
        \xi_i^\top
        \E[\xi_j| z_j]
    \right]
    =0$,
where the last equality follows from~\Cref{eq:xi_centered}.
The variance identity then follows by expanding $\|e_Q\|_2^2$ and using the orthogonality above.
\end{proof}
\Cref{lem:decomp} is the step that makes the problem tractable. The residual left by $Q$ is a difference sequence in the action index, so its variance carries \textbf{no cross terms} even though the underlying gradient contributions are strongly coupled through $\theta$.
Then, we are ready to derive our main theorem: the variance-optimal correction in closed form.
\begin{theorem}\label{thm:opt_cv}
    Let $z_j = b_{<j}$. For any control variates $\Delta = (\Delta_1, \dots, \Delta_d)$ in~\Cref{eq:centered_rho},
    \begin{align}\label{eq:var_loss}
        \Var(\ghat{\Delta}) = \Var(\ghat{Q}) + \sum_{j=1}^d \E[\|\Tilde \Delta_j\|^2 - 2 \xi_j^\top \Tilde \Delta_j] = \sum_{j=1}^d \E \|\xi_j(z_j, b_j) - \Tilde \Delta_j(z_j, b_j)\|^2.
    \end{align}
    Hence, among unrestricted control variates,  the optimal one is
        $\Tilde \Delta_j^*(z_j,v) = \xi_j(z_j, v), \forall j \in [d]$, which attains $\Var(\ghat{\Delta^*}) = 0$ (see Remark~\ref{rem:doob}).
\end{theorem}
\begin{proof}
    See~\Cref{app:pf_opt_cv} for proof. The key idea is that Lemma~\ref{lem:decomp} decomposes the remaining noise into orthogonal action-wise residuals $\xi_j$. The variance therefore separates across $j$ and is minimized by matching each centered control variate to $\xi_j$.
\end{proof}
\vspace{-10pt}
\Cref{thm:opt_cv} shows that, after applying the conditional-mean control variate $Q$, the remaining variance is exactly the mismatch between the residual increment $\xi_j$ and the additional control variate, and that these mismatches add up over decisions with no cross terms. The unrestricted minimizer is a benchmark, not an estimator (Remark~\ref{rem:doob}); all optimality claims below refer to the realizable family of~\Cref{sec:row}. 

\begin{algorithm}[htbp]
\caption{Learning through a black box with an action-conditioned control variate.}
\label{alg:main}
\begin{algorithmic}[1]
\Require policy $p_\theta(b)=\prod_{j=1}^{d}p_j(b_j| z_j)$, black box $f$, fitted reward model $\hat f$, step size $\eta$
\State $\mathcal D\gets\{(b,f(b)):b\sim p_{\theta_0}\}$ \Comment{a fixed offline dataset; $\mathcal D=\emptyset$ also works}
\For{$t=1,2,\dots$}
  \State fit $\hat f$ to $\mathcal D$ \Comment{any regressor, e.g., a neural network}
  \State sample $b\sim p_\theta$ and pay one call $f(b)$ 
  \State $\hat Q(z_j,v)\gets\E\big[\hat f| z_j,b_j=v\big]$ for every $j \in [d]$ and value $v \in \gV$ \Comment{under $p_\theta$, no oracle call}
  \State $\ghat{\hat Q}\gets\sum_{j=1}^{d}\Big[\big(f(b)-\hat Q(z_j,b_j)\big)s_j(b_j)+\sum_v p_j(v)\hat Q(z_j,v)s_j(v)\Big]$ \Comment{\Cref{eq:g_hat_Q}}
  \State $\hat \xi_j\gets\E_{\hat f}[\hat g_{\hat Q}| z_j,b_j]-\E_{\hat f}[\hat g_{\hat Q}| z_j]$ \ for every $j \in [d]$ \Comment{\Cref{eq:xi_def}}
  \State $\hat g\gets\ghat{\hat Q}-\sum_{j,r}\Pi_{j,r}\hat \xi_{j,r}$ \quad 
  \Comment{no matrix solve, fast}
  \State $\theta\gets\theta-\eta\hat g$; \quad $\mathcal D\gets\mathcal D\cup\{(b,f(b))\}$
\EndFor
\end{algorithmic}
\end{algorithm}

\subsection{Variance-Optimal Corrections for Neural Network Gradients}\label{sec:row}
\Cref{thm:opt_cv} characterizes the ideal control variate in gradient space, but does not yet specify a practical representation. For a neural network, we know the direction in which each parameter's gradient can move, and we use this to build a practical correction. Consider a linear weight matrix and one of its parameter rows $r$. At action $j$, the gradient within that row lies along the corresponding layer input. Let $h_{j,r}$ denote this local input direction, which is fixed by the prefix $z_j$ and therefore does not depend on the candidate action $v$. The realizable row-wise subspace is
$\mathcal S_{j,r}\coloneqq \operatorname{span}\{h_{j,r}\}$.
Write $\xi_{j,r}$ variate for the component of $\xi_j$ in row $r$. We therefore restrict the correction for row $r$ to the realizable subspace $\mathcal S_{j,r}$. Within this restricted family,~\Cref{eq:var_loss} shows that the variance-minimizing correction is a point in $\mathcal S_{j,r}$ closest to $\xi_{j,r}$. The remaining question is therefore how to compute this closest realizable correction, which is answered by the following theorem.
\begin{theorem}\label{thm:row_opt}
Define an operator $\Pi_{j,r}(u)\coloneqq{ h_{j,r}h_{j,r}^\top u}/{\|h_{j,r}\|_2^2}$, with $\Pi_{j,r}(u)=0$ when $h_{j,r}=0$. Within the row-wise family, the variance minimization separates across $j$ and $r$; moreover, the optimal correction is automatically centered, i.e.,
\begin{align}\label{eq:row_var}
    \Tilde \Delta^*_{j, r}(z_j, v) = \Pi_{j, r} \xi_{j, r}(z_j, v), \ \ \Var(\ghat{\Delta^*})=\Var(\ghat{{Q}})-\sum_{j=1}^{d}
    \E\Big[\sum_r\left\|
            \Pi_{j,r}\big(\xi_{j,r}(z_j,b_j)\big)\right\|_2^2\Big].
\end{align}
\end{theorem}
\vspace{-12pt}
See Appendix~\ref{app:pf_row_opt} for proof.~\Cref{thm:row_opt} gives a simple interpretation of the optimal
row-wise correction. Each parameter row removes exactly the component of
the residual noise $\xi_{j,r}$ that lies along a gradient direction that
the row can produce. The orthogonal component cannot be removed within
this row-wise family, while the removable variance is precisely the
squared norm of the projection in~\Cref{eq:row_var}. Thus, the global
variance-minimization problem reduces to independent one-dimensional
projections, with no additional control-variate network or coupled
covariance solve.~\Cref{alg:main} summarizes the resulting procedure.

\section{Experiment Settings}
\Cref{sec:method} shows that the action-value function $Q$ removes decision-wise noise, while the correction $\Delta$ targets the remaining variance across decisions. We evaluate how much this additional correction helps in practice on three black-box problems: 1) a quadratic program (QP) from DFL, 2) a live LLM agent whose tools are selected by a language model, and 3) job-shop scheduling, where an LLM generates dispatching rules and a solver constructs the final schedule. Unless otherwise stated, results are averaged over multiple random seeds, and shaded regions indicate one standard deviation. See more experimental details in~\Cref{app:add_exp}.

\paragraph{Baselines.} We compare our approach with different methods:
\begin{itemize}[nosep,left=0.5em]
    \item \textbf{REINFORCE}~\citep{williams1992simple}: the vanilla score-function estimator with the same reward for all decisions.
    \item \textbf{State baseline $V$}: an action-independent prefix baseline,
    $V_j(z_j)=\sum_v p_j(v)Q_j(z_j,v)$.
    \item \textbf{RLOO}~\citep{kool2019buy}: uses rewards of other samples in the batch as a leave-one-out baseline.
    \item \textbf{GRPO}~\citep{shao2024deepseekmath}: normalizes each reward within the sampled group.
    \item \textbf{OTB}~\citep{li2026the}: uses the optimal action-independent baseline at each token.
    \item \textbf{RELAX}~\citep{grathwohl2018backpropagation}: a variance-trained control variate with a pathwise correction (LAX for the Gaussian predictor, RELAX for the discrete ones).
\end{itemize}

\subsection{Portfolio Optimization (DFL)}
We first consider a DFL problem in which multiple predictions share the same network and jointly determine a downstream decision.
A shared network reads daily features and predicts the returns of $d=12$ assets.
Given a sampled return vector $b$, the \emph{black-box oracle} solves a risk-averse portfolio optimization problem to obtain a portfolio $w^*(b)$, evaluates this portfolio under the true returns, and returns only a scalar normalized reward.
The reward is the portfolio's mean-variance utility under the true returns, rescaled separately for each day. A reward of $1$ corresponds to the portfolio obtained from the true returns, while $0$ corresponds to the portfolio obtained from sign-flipped returns.

The predictor outputs a Gaussian distribution with fixed standard deviation, so the sum over $v$ in~\Cref{eq:g_hat_Q} becomes a Gaussian expectation that is available in closed form for a quadratic reward model.
We use an offline dataset of predicted returns and oracle rewards to fit this quadratic model and update during policy training.
The fitted model provides $\hat Q_j$ and the corresponding control variates without additional oracle calls.
The action space is continuous and the downstream reward varies smoothly. Our method extends to continuous actions because its derivation only requires expectations under the action distribution, which are available in closed form here under the Gaussian policy and quadratic reward model.

\subsection{Live MCP Agent}
We next consider a discrete black-box agent setting in which the reward is
observed only after an autoregressive configuration has been executed.
An LLM reads a task and generates the tool configuration for another black-box agent.
The configuration has $4$ slots, and each slot selects one of three candidate MCP servers or no server. The candidates are drawn from a pool of $12$ servers (e.g., filesystem, git, sqlite, and pytest). We allow at most $2$ of the $4$ slots to contain servers, while the remaining slots are empty, to avoid attaching unnecessary tools.
The agent then attempts the task using the selected servers, and the black box
returns $1$ only if the task is completed and every attached server is used.
Thus, one decision corresponds to one slot, while the reward is observed only
after the full configuration has been executed.

For the LLM predictor, we use the Llama-3.2-3B-Instruct model~\citep{grattafiori2024llama} with rank-$8$
LoRA~\citep{hu2022lora}, and for the black-box oracle, we run through OpenAI's GPT 4.1-mini API. For each task, only $1$--$3$ of the $256$ configurations succeed.
We use the four tasks that the base agent can solve and cache the reward of
every configuration, allowing us to run $20$ seeds and compute the exact
$P(\mathrm{success})$ throughout training.

\subsection{Job-Shop Scheduling}
Finally, we consider a combinatorial black-box setting where the reward is observed only after a complete dispatching plan is executed.
An LLM reads a job-shop instance and generates a dispatching plan.
The machines are partitioned into $4$ groups, and the plan assigns one of $5$ dispatch rules to each group, giving $5^4=625$ possible plans.
Given a plan, the {black-box oracle} constructs the schedule by selecting ready operations according to each machine group's rule.
It returns the instance lower bound divided by the resulting makespan, i.e., the total time to complete all jobs, giving a graded reward between $0$ and $1$.
Thus, one decision selects one group's rule, while the reward is returned only after the full schedule is constructed.
The four rules are generated autoregressively as a single JSON plan, so earlier choices can affect later ones.

\section{Discussion of Experimental Results}
\begin{figure}[htbp]
    \centering
    \includegraphics[width=0.9\linewidth]{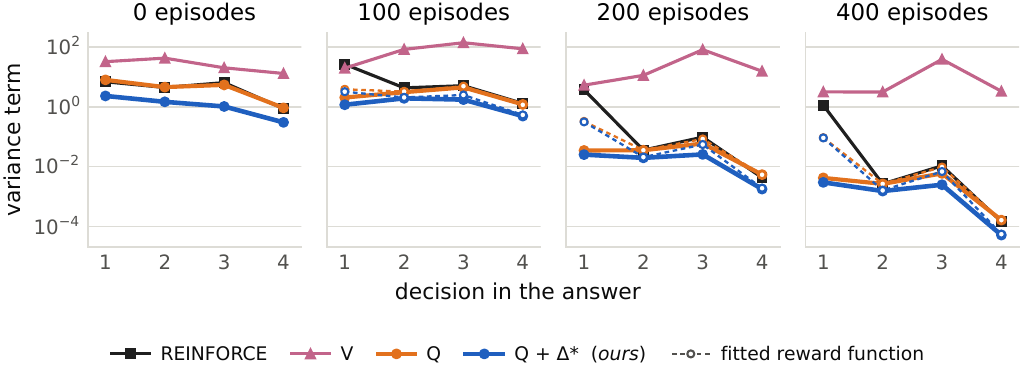}
    \vspace{-12pt}
    \caption{Variance across decisions in the LLM MCP task at different stages of training. The $x$-axis indexes the 4 decision slots, and the $y$-axis reports the corresponding variance on a log scale. Solid curves use the true reward table, while dotted curves use the fitted reward model.}
    \label{fig:decomp}
    \vspace{-20pt}
\end{figure}
\paragraph{When will we need an action-dependent baseline and where does the correction reduce variance for action-dependent baselines?}
Previously, \citet{tucker2018mirage} show that in practical RL settings, action-dependent baselines provide limited gains over action-independent ones, suggesting that their benefit can depend strongly on the problem structure.
Here, our decomposition provides a complementary view by identifying the variance that remains after action-value baseline has solved the local credit-assignment problem.

\Cref{fig:decomp} directly tests the decomposition in~\Cref{lem:decomp}. Each panel shows a different stage of training. It shows that the action-dependent baseline $Q$ consistently reduces variance compared to the action-independent baseline $V$, indicating that local credit assignment matters in this task. Second, substantial variance can remain after conditioning on the current action, and our correction $Q+\Delta^*$ further reduces this residual, which achieves the lowest variance across decisions and training stages. Third, the same ordering largely persists when $Q$ and the correction are built from a fitted reward model rather than the true reward table, which suggests that the gains are not restricted to oracle reward information.

\paragraph{Does the variance reduction improve learning?}
We next ask whether the variance reductions in \Cref{fig:decomp} are reflected in optimization performance. Across all three tasks in \Cref{fig:learning_curve}, the action-dependent baseline $Q$ learns faster and reaches better solutions than REINFORCE and the standard baselines, while adding our correction yields a further improvement and typically reaches strong solutions with fewer black-box evaluations. The effect is particularly large on the QP and MCP tasks, where $Q+\Delta^*$ approaches the exact-gradient reference and learns substantially faster. On MCP, the corrected estimator is also more stable under aggressive optimization, whereas higher-variance estimators can collapse to policies from which successful trajectories are rarely observed. The gain is smaller on JobShop, where the graded reward provides a denser learning signal, but the same qualitative ordering remains. Overall, the reduction in residual variance observed in \Cref{fig:decomp} is consistently accompanied by faster and more stable learning.
\begin{figure}[htbp]
    \centering
    \includegraphics[width=0.98\linewidth]{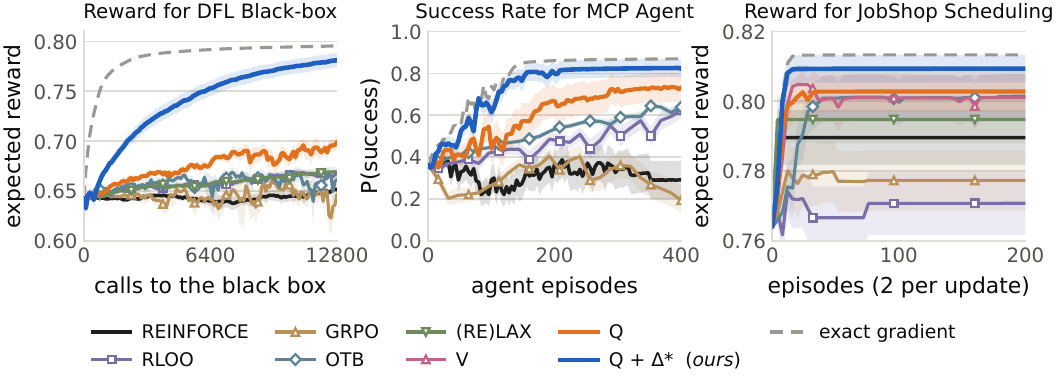}
    \vspace{-12pt}
    \caption{Training performance on three black-box tasks. We compare our methods with baselines. The exact gradient curve shows training with the exact gradient of the expected black-box objective and is included only as an oracle reference.}
    \label{fig:learning_curve}
    \vspace{-12pt}
\end{figure}





\begin{wrapfigure}{r}{0.52\textwidth}
  \centering
  \vspace{-0.2cm}
    \includegraphics[width=\linewidth]{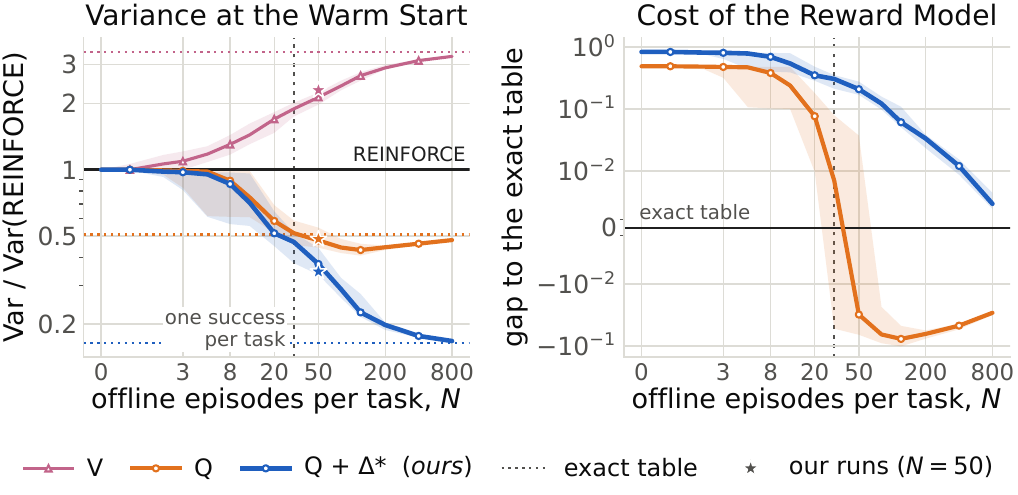}
    \vspace{-0.6cm}
    \caption{More offline data improves the fitted reward model and strengthens variance reduction.}
    \label{fig:offline_data}
\end{wrapfigure}
\paragraph{How much offline data is needed?} \Cref{fig:offline_data} varies the data used to fit the reward model. With too few episodes, the fitted baselines provide little benefit because successful trajectories are rarely observed. Once the budget reaches roughly one expected success per task, both $Q$ and $Q+\Delta^*$ reduce variance substantially, well before the reward model is exact. The right panel shows that fitted $Q$ can even outperform exact $Q$, since exact $Q$ is only coordinate-wise optimal, whereas the gap for $Q+\Delta^*$ remains positive and approaches zero as the reward model improves.


\section{Conclusion}
We show that the action-value function can remain suboptimal for the full gradient when decisions share parameters. By decomposing the remaining variance across decisions, we derive a closed-form correction that reduces to simple row-wise projections for neural-network parameters. Across three black-box tasks, the correction consistently reduces variance beyond the action-value function and improves learning. Overall, variance reduction should target the full shared-parameter gradient rather than each decision in isolation.

\subsection*{AI use statement}
In this work, we used generative AI tools for assisting in the writing of proofs, designing or providing feedback on research methodology or experiments, implementing methods.
We have not used generative AI tools for generating synthetic data sets, helping develop theoretical models or conceptual frameworks, formulating mathematical claims, providing critical ingredients for proving mathematical claims, proposing or refine hypotheses, cleaning and reformatting dataset, supporting qualitative and thematic data analysis, interpret results.
and assisting with translation are not applicable to this work.
Additionally, we used generative AI tools for creating or modifying scientific figures or images,  creating or editing software coding, and editing a research paper to improve readability. We have reviewed all AI-assisted work. We take responsibility for the final content of this work,
including text, claims or artifacts produced with the aid of generative AI.





\subsection*{Reproducibility statement}
We release an open-source repository containing all code, configuration files, and scripts needed to reproduce our results, including data generation and figure plotting. All proofs for the main paper are stated in the appendix with explanations and proper assumptions.



\bibliography{iclr2027_conference}
\bibliographystyle{iclr2027_conference}

\appendix
\section{Appendix}

\subsection{Proof for~\Cref{lem:q_coord_opt}}\label{app:pf_q_coord_opt}
\begin{proof}
For unbiasedness, conditioning on $z_j$ gives
\begin{align}
    \E_{b_j|z_j}
    \big[
        \hat Q_j(z_j,b_j)s_j(b_j)
    \big]
    =
    \sum_{v\in\mathcal V}
    p_j(v)\hat Q_j(z_j,v)s_j(v),
\end{align}
so the two terms involving $\hat Q_j$ cancel in expectation.

For the second claim, consider any perturbation $Q_j+\delta_j$. Let
$\ghat{Q,j}$ denote the $j$-th summand obtained with the exact $Q_j$, and define
\begin{align}
    u_j
    :=
    \delta_j(z_j,b_j)s_j(b_j)
    -
    \sum_v p_j(v)\delta_j(z_j,v)s_j(v).
\end{align}
Then
\begin{align}
    \ghat{Q+\delta,j}
    =
    \ghat{Q,j}-u_j,
    \qquad
    \E[u_j|z_j]=0.
\end{align}
At the exact conditional mean,
\begin{align}
    \ghat{Q,j}
    -
    \E[\ghat{Q,j}|z_j]
    =
    \big(f(b)-Q_j(z_j,b_j)\big)s_j(b_j).
\end{align}
Since $\E[f-Q_j|z_j,b_j]=0$, this residual is conditionally orthogonal to $u_j$. Hence
\begin{align}
    \Var\big(\ghat{Q+\delta,j}|z_j\big)
    =
    \Var\big(\ghat{Q,j}|z_j\big)
    +
    \E\big[\|u_j\|_2^2|z_j\big],
\end{align}
which is minimized by the exact conditional mean (up to perturbations that induce the same zero correction).
\end{proof}

\subsection{Arbitrarily worse}\label{app:arbitrary}
\Cref{ex:two_bern} shows that an exact conditional value function need not be variance-optimal. This
appendix strengthens the statement in the two directions that matter for the claim in the abstract. First, the
suboptimality can be made arbitrarily large while both variances stay positive. Second, in the worst case,
the exact value function can remove \emph{none} of the REINFORCE variance while the optimal correction
removes \emph{all} of it. Both constructions are autoregressive, so they are instances of the setting that
\Cref{lem:decomp} and \Cref{thm:row_opt} are stated for, and in both the value function is exact rather
than fitted.

\begin{proposition}[The gap is unbounded]\label[proposition]{prop:arbitrary}
Consider the autoregressive setting with the exact conditional value function $Q$,
and let $\ghat{\mathrm{row}}$ be the optimal row-wise estimator of \Cref{thm:row_opt}.
\begin{enumerate}[leftmargin=2em,itemsep=1pt,topsep=2pt]
    \item For every $\rho\in[0,1)$ there is an instance with $d=2$ binary decisions and a single shared
    weight vector in which
    \begin{align*}
        \Var(\ghat{Q})=\tfrac12,\qquad
        \Var(\ghat{\mathrm{row}})=\tfrac12\big(1-\rho^{2}\big),\qquad\text{so}\qquad
        \frac{\Var(\ghat{Q})}{\Var(\ghat{\Delta})}=\frac{1}{1-\rho^{2}},
    \end{align*}
    where $\rho$ is the cosine between the two decisions' input directions. Both variances are strictly
    positive and the ratio is unbounded as $\rho\to1$.
    \item For every $d\ge2$ there is an instance with $d$ binary decisions and one shared parameter in which
    \begin{align*}
        \Var(\ghat{\mathrm{SF}})=\Var(\ghat{Q})=\frac d4,
        \qquad \Var(\ghat{\Delta^*})=0 .
    \end{align*}
    The exact value function removes no variance at all, for any horizon, while the optimal correction of the
    same family removes all of it.
\end{enumerate}
\end{proposition}
\begin{proof}
At the stated parameter values every decision has $p_j(1)=\tfrac12$, so $s_j(v)=(v-\tfrac12)\,h_j$ in part~1 and $s_j(v)=v-\tfrac12$ in part~2. We use the identity $(2v-1)(v-\tfrac12)=\tfrac12$ for $v\in\{0,1\}$ throughout.

\emph{Part 1.} With prefix conditioning, the exact conditional means are
\begin{align*}
    Q_1(v)=\E[f\mid b_1=v]=(2v-1)\,\E[2b_2-1]=0,
    \qquad
    Q_2(b_1,v)=(2b_1-1)(2v-1)=f(b_1,v).
\end{align*}
In~\Cref{eq:g_hat_Q}, the correction term of $j=1$ vanishes because $Q_1\equiv0$, and term~(a) of $j=2$ vanishes because $Q_2=f$. The two surviving pieces are
\begin{align*}
    f(b)\,s_1(b_1)&=(2b_1-1)(2b_2-1)\big(b_1-\tfrac12\big)\,h_1=\tfrac12(2b_2-1)\,h_1,\\
    \sum_v \tfrac12\,Q_2(b_1,v)\,s_2(v)&=\tfrac12(2b_1-1)\sum_v (2v-1)\big(v-\tfrac12\big)\,h_2=\tfrac12(2b_1-1)\,h_2,
\end{align*}
so that
\begin{align*}
    \ghat{Q}=\tfrac12(2b_2-1)\,h_1+\tfrac12(2b_1-1)\,h_2,
    \qquad
    g=\E[\ghat{Q}]=0.
\end{align*}
The residual increments of~\Cref{eq:xi_def} are therefore
\begin{align*}
    \xi_1(b_1)=\E[\ghat{Q}\mid b_1]-g=\tfrac12(2b_1-1)\,h_2,
    \qquad
    \xi_2(b_1,b_2)=\ghat{Q}-\E[\ghat{Q}\mid b_1]=\tfrac12(2b_2-1)\,h_1,
\end{align*}
and \Cref{lem:decomp} gives $\Var(\ghat{Q})=\E\|\xi_1\|^2+\E\|\xi_2\|^2=\tfrac14+\tfrac14=\tfrac12$. Note that $\xi_1$ lies along $h_2$, the input direction of the \emph{other} decision: the noise left at decision~1 is entirely the effect of $b_1$ on decision~2's gradient term.

There is a single row, so $\Pi_j=h_jh_j^\top$ and
\begin{align*}
    \Pi_1\xi_1=\tfrac{\rho}{2}(2b_1-1)\,h_1=\rho\,s_1(b_1),
    \qquad
    \Pi_2\xi_2=\tfrac{\rho}{2}(2b_2-1)\,h_2=\rho\,s_2(b_2),
    \qquad
    \|\Pi_1\xi_1\|^2=\|\Pi_2\xi_2\|^2=\tfrac{\rho^2}{4}.
\end{align*}
\Cref{thm:row_opt} then yields
\begin{align*}
    \Var(\ghat{\mathrm{row}})=\tfrac12-\tfrac{\rho^2}{4}-\tfrac{\rho^2}{4}=\tfrac12\,(1-\rho^2),
    \qquad
    \ghat{\mathrm{row}}=\ghat{Q}-\rho\,s_1(b_1)-\rho\,s_2(b_2).
\end{align*}
The optimal row-wise correction is thus a baseline of value $\rho$ on each decision on top of $Q$. Its value is set by the cross-decision alignment; the classical per-decision baseline criterion would set it to zero here, since $\E[f-Q_j\mid z_j,b_j]=0$.

\emph{Part 2.} Write $P\coloneqq\prod_{k<d}(2b_k-1)\in\{\pm1\}$. Because the factors $2b_k-1$ are independent with mean zero,
\begin{align*}
    Q_j(b_{<j},v)=\prod_{k<j}(2b_k-1)\,(2v-1)\prod_{k>j}\E[2b_k-1]=0\quad (j<d),
    \qquad
    Q_d(b_{<d},v)=P\,(2v-1).
\end{align*}
For REINFORCE, $f^2\equiv1$ and $\E[f(b)(b_j-\tfrac12)]=0$ for $d\ge2$, so
\begin{align*}
    \Var(\ghat{\mathrm{SF}})=\E\Big[\Big(\sum_{j=1}^d\big(b_j-\tfrac12\big)\Big)^2\Big]=\frac d4 .
\end{align*}
For $\ghat{Q}$, each summand $j<d$ reduces to $f(b)\,s_j(b_j)$ (both $Q_j$ and its correction term vanish), while the summand $j=d$ reduces to its correction term $\sum_v\tfrac12\,P\,(2v-1)(v-\tfrac12)=\tfrac P2$. Hence
\begin{align*}
    \ghat{Q}
    =f(b)\sum_{j<d}\big(b_j-\tfrac12\big)+\frac P2
    =P\Big[(2b_d-1)\sum_{j<d}\big(b_j-\tfrac12\big)+\frac12\Big],
    \qquad \E[\ghat{Q}]=0=g,
\end{align*}
and, since $P^2=(2b_d-1)^2=1$ and $b_d$ is independent of $b_{<d}$ with $\E[2b_d-1]=0$,
\begin{align*}
    \E[\ghat{Q}^2]
    =\E\Big[\Big(\sum_{j<d}\big(b_j-\tfrac12\big)\Big)^2\Big]
    +\E\Big[(2b_d-1)\sum_{j<d}\big(b_j-\tfrac12\big)\Big]
    +\frac14
    =\frac{d-1}{4}+0+\frac14=\frac d4 .
\end{align*}
Finally, $\theta$ is a scalar, so every $h_{j,r}=1$ and the row-wise family is unrestricted; \Cref{thm:row_opt} with~\Cref{eq:vq_decomp} gives $\Var(\ghat{\mathrm{row}})=\Var(\ghat{Q})-\sum_j\E\|\xi_j\|^2=0$.
\end{proof}

\subsection{How can we compute the $\xi_j$?}
\begin{lemma}\label[lemma]{lem:xi}
Let $\ghat{Q,>j}:=\sum_{i>j}\ghat{Q,i}$. For an autoregressive model, revealing $b_j$
can affect gradient terms both before and after $j$, i.e.,
\begin{align}\label{eq:xi_decomp}
    \xi_j(z_j,v)
    =
    (
        Q(z_j,v)
        -
        \sum_w p_j(w)Q(z_j,w)
    )
    \sum_{i<j}s_i(b_i)
    +
    \E[\ghat{Q,>j}| z_j,b_j=v]
    -
    \E[\ghat{Q,>j}| z_j].
\end{align}
\end{lemma}
\begin{proof}
We decompose $\ghat{Q}$ by time step. {For $i<j$, $s_i(b_i)$ and $Q_i(\ctx{i},b_i)$ are determined by $b_{\le j}$, and the $j$-th correction term $\sum_w p_j(w)Q_j(z_j,w)s_j(w)$ is determined by $z_j$; the only remaining random object in the summands $i\le j$ is $f(b)$, whose conditional mean is $Q_j(v)$ by~\Cref{eq:Q}. Conditioning the summands $i\le j$ of~\Cref{eq:g_hat_Q} on $(z_j,b_j=v)$ therefore gives}
\begin{align}
{\E\Big[\sum_{i\le j}\ghat{Q,i}\,\Big|\,\ctx{j},b_j=v\Big]=\;}Q_j(v)\underbrace{\sum_{i<j}s_i(b_i)}_{S_{<j}} \underbrace{-\sum_{i<j}Q_i(\ctx{i},b_i)s_i(b_i)+\sum_{i\le j}\sum_w p_i(w)Q_i(\ctx{i},w)s_i(w)}_{\coloneqq R_j \ \text{(does not depend on $v$})},
\end{align}
{where term~(a) of the $j$-th summand has zero conditional mean and was dropped.} Thus, we can rewrite~\Cref{eq:g_hat_Q} as
\begin{align}
    \E[\hat g_Q | \ctx{j},b_j=v] = R_j + Q_j S_{<j} + D_j,
\end{align}
where $D_j(v)=\mathbb{E}[\sum_{i>j}((f(b)-Q_i(\ctx{i},b_i))s_i(b_i)+\sum_{w}p_i(w)Q_i(\ctx{i},w)s_i(w))|\ctx{j},b_j=v]$ collects the remaining conditional expected gradient after choosing $v$.
Similarly, we have
\begin{align}
    \E[\hat g_Q | \ctx{j}] = R_j + V_j S_{<j} + \bar D_j,
\end{align}
where $\bar D_j=\sum_v p_j(v)D_j(v)$. Thus, we have 
\begin{align}
    \xi_j(\ctx{j}, v) \coloneqq \E[\ghat{Q} | \ctx{j}, b_j = v] - \E[\ghat{Q} | \ctx{j}] = (Q_j - V_j) S_{<j} + D_j - \bar D_j.
\end{align}
\end{proof}
{The first term, $(Q_j-V_j)S_{<j}$, arises because observing $b_j$ changes the credit assigned to the earlier score terms; the second, $D_j-\bar D_j$, because $b_j$ changes the distribution, and hence the expected gradient contributions, of the later decisions. $Q$ removes the randomness local to decision $j$ (there is no $i=j$ term), but both cross-decision effects remain. Note that both terms are linear combinations of scores $s_i(\cdot)$ with $i\neq j$, which is what \Cref{cor:orth} uses.}

\subsection{Proof for~\Cref{thm:opt_cv}}\label{app:pf_opt_cv}
\begin{proof}
Since $\ghat{\Delta}=\ghat{Q}-\sum_j \Tilde\Delta_j$ is unbiased,
\begin{align}
    \Var(\ghat{\Delta})
    &=
    \E[
        \|e_Q-\sum_{j=1}^d \Tilde\Delta_j\|_2^2
    ]
    \notag\\
    &=
    V_Q
    -2\sum_{j=1}^d
    \E[
        e_Q^\top \Tilde\Delta_j
    ]
    +
    \E[
        \|\sum_{j=1}^d \Tilde\Delta_j\|_2^2
    ].
    \label{eq:cv_expand}
\end{align}

For $i<j$, $\Tilde\Delta_i$ is determined by $z_j=b_{<j}$, while
$\E[\Tilde\Delta_j\mid z_j]=0$. Hence
\begin{align}
    \E[
        \Tilde\Delta_i^\top
        \Tilde\Delta_j
    ]
    =
    \E[
        \Tilde\Delta_i^\top
        \E[\Tilde\Delta_j\mid z_j]
    ]
    =0,
\end{align}
so
\begin{align}
    \E[
        \|\sum_{j=1}^d \Tilde\Delta_j\|_2^2
    ]
    =
    \sum_{j=1}^d
    \E[
        \|\Tilde\Delta_j\|_2^2
    ].
    \label{eq:delta_orth}
\end{align}

Next, since $\Tilde\Delta_j$ depends only on $(z_j,b_j)$,
\begin{align}
    \E[
        e_Q^\top \Tilde\Delta_j
    ]
    &=
    \E[
        \E[e_Q\mid z_j,b_j]^\top
        \Tilde\Delta_j
    ].
\end{align}
By the definition of $\xi_j$,
\begin{align}
    \E[e_Q\mid z_j,b_j]
    =
    \xi_j(z_j,b_j)
    +
    \E[\ghat{Q}\mid z_j]
    -
    g.
\end{align}
The last two terms depend only on $z_j$, and therefore vanish after taking
the inner product with $\Tilde\Delta_j$. Thus,
\begin{align}
    \E[
        e_Q^\top \Tilde\Delta_j
    ]
    =
    \E[
        \xi_j^\top \Tilde\Delta_j
    ].
    \label{eq:xi_delta}
\end{align}

Substituting~\Cref{eq:delta_orth,eq:xi_delta} into~\Cref{eq:cv_expand} gives
\begin{align}
    \Var(\ghat{\Delta})
    =
    V_Q
    +
    \sum_{j=1}^d
    \E[
        \|\Tilde\Delta_j\|_2^2
        -
        2\xi_j^\top \Tilde\Delta_j
    ].
\end{align}

Finally, using
\begin{align}
    V_Q
    =
    \sum_{j=1}^d
    \E[
        \|\xi_j\|_2^2
    ],
\end{align}
we obtain
\begin{align}
    \Var(\ghat{\Delta})
    =
    \sum_{j=1}^d
    \E[
        \|\xi_j-\Tilde\Delta_j\|_2^2
    ].
\end{align}
Therefore, the unrestricted minimum is attained at
\begin{align}
    \Tilde\Delta_j^*(z_j,v)
    =
    \xi_j(z_j,v),
\end{align}
for each $j$.
\end{proof}

\subsection{Proof for~\Cref{thm:row_opt}}\label{app:pf_row_opt}
\begin{proof}
For a fixed $(j,r)$,~\Cref{eq:var_loss} reduces the optimization to
\begin{align*}
    \arg\min_{u\in\mathcal S_{j,r}}
    \|u-\xi_{j,r}(z_j,v)\|_2^2,
\end{align*}
whose solution is the Euclidean projection
\begin{align*}
    \Tilde\Delta_{j,r}^*(z_j,v)
    =
    \Pi_{j,r}\xi_{j,r}(z_j,v).
\end{align*}

Since $\Pi_{j,r}$ depends on the prefix and row but not on $v$,
\begin{align*}
    \sum_v p_j(v)\Tilde\Delta_{j,r}^*(z_j,v)
    &=
    \Pi_{j,r}
    \sum_v p_j(v)\xi_{j,r}(z_j,v)=0,
\end{align*}
where the last equality follows from~\Cref{eq:xi_centered}.
Thus, the projected correction is already centered.

Finally, by the Pythagorean identity,
\begin{align*}
    \|\xi_{j,r}-\Pi_{j,r}\xi_{j,r}\|_2^2
    =
    \|\xi_{j,r}\|_2^2
    -
    \|\Pi_{j,r}\xi_{j,r}\|_2^2.
\end{align*}
Summing over rows and decisions and using~\Cref{eq:vq_decomp}
gives the stated variance formula.
\end{proof}

\subsection{Additional Remarks}\label[remark]{app:remarks}
\begin{remark}[The unrestricted optimum is only a benchmark]\label{rem:doob}
Setting $\Tilde\Delta_j=\xi_j$ removes all remaining noise in $\ghat{Q}$ and therefore gives zero variance. However, computing $\xi_j$ requires exact conditional expectations of the black-box reward, so this choice is not practical. We therefore use it only as an oracle benchmark.

In practice, the correction is restricted by the network gradient structure. Under this restriction, \Cref{thm:opt_cv} shows that the global variance-minimization problem decomposes into simple row-wise problems, leading to \Cref{thm:row_opt}. All optimality claims in this paper refer to this realizable family.
\end{remark}

\begin{corollary}[When $Q$ is already optimal, and a diagnostic]\label[corollary]{cor:orth}
Suppose that, for every row $r$ and every prefix, the input direction $h_{j,r}$ of decision $j$ is orthogonal to the input directions $h_{i,r}$ of all other decisions $i\neq j$. Then $\Pi_{j,r}\xi_{j,r}=0$ for all $(j,r)$, so $\Var(\ghat{\Delta^*})=\Var(\ghat{Q})$. This means $Q$ is variance-optimal within the row-wise family. This holds in particular for disjoint per-decision parameters, recovering \citet{tokui2017evaluating} and the orthogonal-score condition of \citet{wu2018variance}; conversely, in \Cref{prop:arbitrary}(1) the removed fraction of $\Var(\ghat{Q})$ is $\rho^2$.
\end{corollary}
\begin{proof}
By \Cref{lem:xi}, $\xi_j$ is a linear combination of scores $s_i(\cdot)$ of the other decisions $i\neq j$ (averaged over continuations for $i>j$), whose row-$r$ components lie in $\operatorname{span}\{h_{i,r}\}$; under the orthogonality assumption their projections onto $\mathcal S_{j,r}$ vanish, so $\Pi_{j,r}\xi_{j,r}=0$ and \Cref{eq:row_var} gives $\Var(\ghat{\Delta^*})=\Var(\ghat{Q})$. When decisions use disjoint parameters, $h_{j,r}$ and $h_{i,r}$ are supported on disjoint coordinate blocks and are therefore orthogonal. The second claim is computed in the proof of \Cref{prop:arbitrary}.
\end{proof}

\begin{remark}[Computational cost]\label[Remark]{rem:cost}
Line~7 of \Cref{alg:main} requires $\E_{\hat f}[\ghat{\hat Q}\mid z_j,b_j=v]$ for every $(j,v)$, which by \Cref{lem:xi} involves an expectation over the continuations $b_{>j}$ under $p_\theta$. In our experiments the continuations are enumerated exactly (at most $625$ sequences, or a closed-form Gaussian--quadratic expectation), so the correction costs no oracle call but $\gO(d|\gV|)$ evaluations of the fitted model per step, and the enumeration grows exponentially with $d$. For long horizons one can instead draw, independently of the realized $b$, a single continuation per $(j,v)$ under $p_\theta$ and use its value of $\ghat{\hat Q}$ in place of the expectation, at $\gO(d|\gV|)$ model rollouts per step. Because the same sampled continuations enter both terms of \Cref{eq:centered_rho}, the resulting $\Tilde\Delta_j$ still has zero mean given $z_j$, so $\hat g$ remains unbiased; only the variance reduction degrades. We leave this variant to future work.
\end{remark}

\section{Additional Experimental Details}\label{app:add_exp}

\paragraph{The exact gradient.}
The dashed line in every training figure is $\nabla_\theta\E[f]$ itself. On the QP the regret is a quadratic in the prediction, so the Gaussian expectation is closed-form; on the agent there are $256$ configurations per task and the reward of each is cached, so we enumerate them exactly. It is the ceiling of every method, and it costs no calls. The same enumeration gives the exact $\Var(\hat g)$ of every estimator at the current policy, which is what \Cref{fig:decomp} plots, and the exact $P(\text{success})$, which is what we report instead of a sampled success rate.

\paragraph{Portfolio Optimization from DFL} The offline dataset has $300$ calls, $50$ per day. The reward model of each day is a quadratic in $b$ fitted by least squares to the offline calls and the calls paid on that day, with the ridge picked by leave-one-out cross-validation. LAX starts its surrogate from the same dataset; the other baselines have no use for it. We then train with $12{,}800$ calls and $10$ seeds. RLOO, GRPO and OTB draw eight samples per step, and LAX refits its surrogate on the day's last $128$ samples.

\paragraph{Job-shop task and synthetic data}
\label{app:jobshop}
Each of $n$ jobs visits all $m$ machines once in a fixed order. Operations are non-preemptive, and each machine processes at most one operation at a time. The goal is to minimize the makespan, when all jobs finish. Synthetic instance families are summarized in Table~\ref{tab:jobshop_generation}.
Machines are divided into four groups by $g(k)=\lfloor4k/m\rfloor$, where $k$ is the zero-based machine index. The planner selects one of five rules per group, yielding $5^4=625$ plans:
\begin{itemize}[nosep]
    \item \texttt{est\_spt}: earliest start, then shortest processing time.
    \item \texttt{job\_ready}: earliest job availability.
    \item \texttt{ratio}: smallest $(e+1)/(w+1)$, where $e$ is the earliest start and $w$ is the job's remaining processing time.
    \item \texttt{fifo}: smallest next-operation index.
    \item \texttt{completion}: earliest prospective completion.
\end{itemize}
At each step, the constructor scores each job's next unscheduled operation using its machine group's rule and selects the smallest priority key. Ties use job index; postprocessing is disabled.

The reward is the workload lower bound divided by the schedule's makespan. This bound is the larger of the maximum total processing time of any job and the maximum workload on any machine. Higher reward indicates a shorter schedule relative to this bound, without requiring the optimal makespan to be known.

\begin{table}[!htbp]
    \centering\small
    \caption{Synthetic job-shop families. Integer ranges are inclusive; draws are independent except for the shared quantities specified below.}
    \label{tab:jobshop_generation}
    \begin{tabular}{p{0.19\linewidth}p{0.73\linewidth}}
        \toprule
        Family & Generation rule \\
        \midrule
        Uniform & Uniform random machine permutation per job; durations uniform in $1$--$99$. \\
        \addlinespace
        Bottleneck & As Uniform, but multiply durations on one uniformly chosen machine by $4$. \\
        \addlinespace
        Heterogeneous & Uniform routes; durations uniform in $1$--$25$, multiplied by a shared per-job factor uniform in $\{1,2,4\}$ and by a further $4$ independently with probability $0.15$ per operation. \\
        \addlinespace
        Correlated routes & One shared uniform random route, with $\max(1,\lfloor m/5\rfloor)$ uniformly sampled swaps of distinct positions per job; durations uniform in $1$--$99$. \\
        \bottomrule
    \end{tabular}
\end{table}

Each split contains $128$ distinct instances: four per size--family combination across eight sizes and four families. Training, validation, and test share these settings. All instances are checked for positive integer durations and valid machine routes.



\section{Additional Related Works}
\paragraph{Decision-focused learning.}
DFL trains the predictor with the loss of the induced decision~\citep{elmachtoub2022smart}. For more detailed introduction to DFL, we refer the readers to~\citep{mandi2024decision}. Most DFL methods differentiate through the optimization problem~\citep{amos2017optnet, donti2017task, wilder2019melding} or through a smoothed version of it~\cite{pogancic2020differentiation, berthet2020learning, niepert2021implicit}, which needs access to the structure of the problem. When the problem is a black-box, a common approach is to learn a differentiable surrogate of the decision loss~\citep{shah2022decision, zharmagambetov2023landscape}. In this work, we also learn a model of the loss, but we only use it to reduce the variance of an unbiased estimator, so its error never introduces a bias but only affects the variance.

\end{document}